\documentclass[letterpaper, 10pt, conference]{./ieeeconf/ieeeconf}

\usepackage{amsmath, amssymb}
\newtheorem{theorem}{Theorem}
\newtheorem{hypothesis}[theorem]{Hypothesis}
\newtheorem{remark}[theorem]{Remark}
\usepackage{graphicx}
\usepackage{booktabs}
\usepackage[hidelinks]{hyperref}

\newcommand{\KK}{\mathcal{K}}
\newcommand{\DD}{\mathcal{D}}
\newcommand{\N}{\mathbb{N}}
\usepackage{xcolor}

\usepackage[capitalize,nameinlink]{cleveref}

\usepackage{float}
\usepackage{algorithm}
\usepackage{algpseudocode}

\title{\LARGE \bf
Infeasible optimization problems and the hierarchical augmented Lagrangian method in imitation learning
}

\author{Roland Andrews, Justin Carpentier, and Ajay Sathya}

\begin{document}

\maketitle
\thispagestyle{empty}
\pagestyle{empty}

\begin{abstract}
Imitation learning (IL) is an effective approach to train complex robotics policies.
Recent works have introduced hard constraints into imitation-learning optimization problems to ensure safety, stability, and robustness of the learned policy.
However, we argue that these constraints are sometimes infeasible, which can lead to unstable or difficult training dynamics.
We study a simple remedy for such situations based on recent theoretical results on the augmented Lagrangian method in infeasible settings.
We show that our approach drives the learned policy toward the solution of a closest-feasible constrained IL problem with desirable properties.
The method is illustrated on a toy driving example with a total-acceleration constraint and pedestrian-safety constraints, a setting in which infeasibility can naturally arise while still allowing a safe learned policy.
\end{abstract}

\section{INTRODUCTION}
Imitation learning (IL) is an effective approach for training complex robotic policies: it uses expert demonstrations to define all or part of the loss function used to train a policy.
IL is particularly effective when a loss function cannot otherwise be specified easily, or when an available loss provides only a weak gradient signal over the relevant regions of the state space. 
Behavioral cloning is a common IL methodology that consists in learning to imitate expert actions \(U^* = (u^*_k)_{ 0\leq k \leq K}\) together with their corresponding states \( \Xi^* = (\xi_k)_{ 0\leq k \leq K}\):
\begin{equation} \label{eq:behavior_cloning}
   \underset{\hat u \in \DD}{\mathrm{minimize}}\quad \mathcal{L}_{U^* , \Xi^*}( \hat u)  + r(\hat u)
\end{equation}
Here \(\DD\) is a subset of a Hilbert space of functions contained in \(L^2\), typically parameterized by a function class such as a family of neural networks, and \(\mathcal{L}\) is a loss that encourages the learned policy to behave similarly to the expert demonstration.
\(r\) is a regularizer that prevents overfitting; for the theoretical analysis we assume that \(r\) is convex in \(\hat u \) and level bounded. 
In practice, \(r\) is often not convex in \(\hat u \). For instance, if \(\hat u \) is represented by a neural network and \(r\) is a squared penalty on its weights, as in weight decay, then \(r\) is convex 
in the parameters of the neural network, but not on the space of neural networks.
 We nevertheless observe in \Cref{sec:toy_car_experiment} that the conclusions suggested by the convex analysis remain informative in practice. 
In the present work we illustrate our results by choosing the loss \( \mathcal{L}_{ U^* , \Xi^*}( \hat u ) = \frac{1}{K} \sum_{k=0}^K \Vert \hat u(\xi_k) - u^*_k \Vert^2\).
This type of behavioral cloning typically suffers from compounding errors along the learned policy's trajectory, but the results of the present work can easily be extended to more 
elaborate loss functions.

Optimization problem \eqref{eq:behavior_cloning} suffers from safety, robustness, and stability issues. 
Indeed, the expert behavior may itself violate constraints that we would like the learned policy to satisfy
 (for instance, a human driver who sometimes exceeds the speed limit or fails to remain perfectly within lane boundaries). 
It may also happen that the expert data are themselves safe and robust, but that an embodiment mismatch, such as differences in morphology, actuation, or dynamics, makes naive imitation unsafe.
One may think, for example, of a robot learning to walk from human walking data.
For these reasons, the recent literature has studied constrained imitation learning:

\begin{align} \label{eq:constrained IL}
   \underset{\hat{u} \in \DD}{\mathrm{minimize}} &\quad  \mathcal{L}_{U^* , \Xi^*}( \hat u ) + r(\hat u) \\
   \text{s.t.} &\quad C(\hat{u}) \in \KK \nonumber
\end{align}

The condition \(C(\hat{u}) \in \KK\) represents constraints added to ensure that the learned policy is safe, stable, and robust.

For the theoretical discussion we assume that \(\DD\) is a convex and closed.
We also assume that the mapping \(C\) and the convex cone \(\KK\) are such that the set
\[
\{(\hat{u},\phi) \mid \phi \in C(\hat{u}) - \KK\}
\]
is convex, closed, and non-empty. When problem \eqref{eq:constrained IL} is feasible, this last condition simply corresponds to stating that the constraints are convex.
We denote by \(\KK^\circ\) the polar cone of \(\KK\).
These assumptions guarantee that \eqref{eq:constrained IL} is a proper, closed, convex problem. Of course the optimization problem is generally non-convex in the 
parameters of the function class that parameterizes \(\hat u \). 

Infeasibility of \eqref{eq:constrained IL} arises when the set \(\{ \hat{u} \mid C(\hat{u}) \in \KK \}\) is empty. We argue in the literature
review of \Cref{sec:related_works}, as well as through the simple yet intuitive example of \Cref{sec:toy_car_experiment}, that infeasibility may naturally arise in many constrained IL settings.

\subsection{Related Works} \label{sec:related_works}
The existing literature justifies the constrained imitation learning setting \eqref{eq:constrained IL}. However, the cited works do not all formulate constrained IL in the same way: some constrain imitation through reward lower bounds, some through Lyapunov or semidefinite certificates, some through control barrier functions, and some through differentiable completion and correction procedures. Infeasibility is therefore a plausible issue in several of these methods, but it is usually not the main object of study.

\cite{wang2021reward} introduces Reward-Constrained Behavior Cloning, which combines a behavior-cloning objective with a lower-bound constraint on expected return,
 or equivalently on an expected Q-value under the learned policy. The goal is not merely to imitate the demonstration, but to preserve desirable behavior patterns from 
 demonstrations while avoiding severe performance degradation. This paper is directly relevant to infeasibility, since the reward threshold is
  user-chosen and the authors explicitly note that ``the constraints may be unsatisfied for a long time, which [\dots] makes the learning process unstable''.
   The feasibility of the constrained problem depends critically on the chosen lower bound as well as on the constantly evolving estimated Q-value.

\cite{huang2021ekmp} proposes a kernelized imitation learning framework for trajectory generation with obstacle avoidance and linear or nonlinear hard constraints. Their method relies on local linearization of nonlinear costs and constraints together with iterative constrained updates using Lagrange multipliers and proximal regularization. Interestingly, their setting matches exactly the convex setting \eqref{eq:constrained IL}.

\cite{yin2021imitation} derives convex stability and safety conditions for certain neural-network controllers in linear time-invariant systems by combining Lyapunov 
arguments with quadratic constraints on the nonlinearities. The resulting constrained imitation learning problem is solved with an ADMM-based algorithm. 
 Since ADMM can be interpreted as a splitting scheme related to augmented Lagrangian methods,
 it is likely that our results may be extendable to their algorithm and ADMM in general, although we do not investigate that question here.
 
In \cite{cosner2022end}, Control Barrier Functions are used to transfer safety guarantees from a robust expert controller to an end-to-end imitation-learned controller.
Their guarantees are formulated in terms of set invariance and input-to-state safety, and the robust expert controller is obtained through a Tunable Robust Optimization Program (TR-OP), 
which is a constrained imitation-learning optimization problem that accounts for matched disturbances and state uncertainty.
  When such safety constraints are combined with additional design requirements, infeasibility or excessive conservatism may arise in practice. 
  Although we do not use their TR-OP, we illustrate a case of infeasibility that would apply to their examples in \Cref{sec:toy_car_experiment}.
  A related direction is \cite{yang2024enhancing}, where a customized control barrier function is itself learned from safety demonstrations and then used for shielding 
  in a learning-from-demonstration pipeline. It is possible that learning the constraints may itself increase the chance of infeasible constraints.

\cite{diehl2022differentiable} proposes a Differentiable Constrained Imitation Learning framework combining three ingredients: a neural network predicts controls, an explicit completion step enforces equality constraints induced by the dynamics, and a gradient-based correction step reduces inequality violations. During training, soft penalty terms are also used on the remaining constraint residuals. Thus, their method is more structured than a plain penalty method. Importantly for our purposes, they explicitly discuss robustness to incorrect constraints that may render the problem infeasible. This is close in spirit to our motivation, although their work does not provide a general convergence theory for infeasible constrained IL problems.

\subsection{Contribution}
Our contribution is to analyze the issue of infeasibility explicitly and to provide a principled behavior when the imposed constrained IL problem is infeasible.
We build on existing work on the inexact augmented Lagrangian method (IALM) and extend the theory to the case of an augmented Lagrangian method with 
several distinct penalty terms instead of a single uniform penalty for all constraints. We show that in this setting, the IALM converges to a well-defined analytical 
``closest feasible problem'' \eqref{eq:closest_feasible_IL2} with desirable properties. Namely, it respects the most important constraints while minimizing the violation of the 
constraints of lesser importance. We illustrate the theory with a toy experiment.

\section{The IALM algorithm for safe infeasible IL}
The behavior of the augmented Lagrangian method (ALM) for infeasible problems was first studied for quadratic programming (QP) in \cite{chiche2016augmented}.
Recent works \cite{dai2023augmented,andrews2025augmented} study the general convex case. 
They show that when the constraints of a convex problem are infeasible, the ALM converges to the ``closest feasible problem,'' 
defined as the problem with the same objective and constraints shifted by the smallest quantity that makes them feasible.
Equivalently, it is the problem with the same objective in which the constraint violation norm is minimized first (see \cite[Equation 34]{andrews2025augmented} for the analytical formulation).

In our setting, we separate the constraints into two types with different penalty sequences, and we show that this allows us to choose which constraints
should be satisfied exactly and which should only have their violation minimized. This is an improvement, demonstrated here in the present setting, over the framework of
\cite{dai2023augmented,andrews2025augmented} where all constraints have the same penalty term.

We rewrite \eqref{eq:constrained IL} by separating the constraints into two groups:
\begin{align} \label{eq:constrained IL 2}
   \underset{\hat{u} \in \DD}{\mathrm{minimize}} &\quad  \mathcal{L}_{U^* , \Xi^*}( \hat u ) + r(\hat u) \\
   \text{s.t.} &\quad C_A(\hat{u}) \in \KK_A \nonumber \\
    &\quad C_B(\hat{u}) \in \KK_B . \nonumber
\end{align}

We use the notation \( \lambda = (\lambda_A, \lambda_B)\), \(\KK = \KK_A \times \KK_B\), and \(C(\cdot) = (C_A(\cdot), C_B(\cdot))\). 

The augmented Lagrangian (dropping constant terms, see \cite[Appendix A]{andrews2025augmented} for details on reformulations and variants of the augmented Lagrangian) is 

\begin{align}
   L_{\gamma_A, \gamma_B}(\hat{u}, \lambda_A, \lambda_B) =\ & f(\hat{u}) + \frac{\gamma_A}{2} \left\Vert  \mathrm{Proj}_{\KK_A^\circ}
      \left( C_A(\hat{u}) - \frac{\lambda_A}{\gamma_A} \right) \right\Vert^2 \nonumber \\
   & + \frac{\gamma_B}{2} \left\Vert  \mathrm{Proj}_{\KK_B^\circ}
      \left( C_B(\hat{u}) - \frac{\lambda_B}{\gamma_B} \right) \right\Vert^2
\end{align}

We write \( u \underset{\varepsilon}{\approx} \underset{\hat{u} \in \DD}{\mathrm{argmin}}\; L_{\gamma_A, \gamma_B}(\hat{u}, \lambda_A, \lambda_B) \) 
to mean that \(u\) approximately minimizes \( L_{\gamma_A, \gamma_B}( \cdot, \lambda_A, \lambda_B) \) with error \(\varepsilon\) in the sense that
\begin{equation}
   L_{\gamma_A, \gamma_B}( u, \lambda_A, \lambda_B) - \underset{\hat{u} \in \DD}{\mathrm{argmin}}\; L_{\gamma_A, \gamma_B}(\hat{u}, \lambda_A, \lambda_B) \leq \varepsilon
\end{equation}

\begin{algorithm}[t]
   \caption{Hierarchical IALM}
   \label{algo:IALM_IL}
   \begin{algorithmic}[1]
   \Require Problem data, initial $\hat{u}^{0}$, $\lambda^{0}$, sequence of positive errors \( (\varepsilon_k)_{i\in \N}\), sequences of  positive penalties \( (\gamma_{A,i} , \gamma_{B,i})_{i\in \N} \)
   \For{$i = 0, 1, \ldots$}
       \State $\hat{u}^{i+1} \underset{\varepsilon_k}{\approx} \arg\min_{\hat{u} \in \DD}\; L_{\gamma_{A,i}, \gamma_{B,i}}(\hat{u}, \lambda_A^i, \lambda_B^i) $
       \State $\lambda_A^{i+1} \gets -\gamma_{A,i} \mathrm{Proj}_{\KK_A^\circ}(C(\hat{u}^{i+1}) - \lambda_A^i/\gamma_{A,i})$ 
       \State $\lambda_B^{i+1} \gets -\gamma_{B,i} \mathrm{Proj}_{\KK_B^\circ}(C(\hat{u}^{i+1}) - \lambda_B^i/\gamma_{B,i})$ 
   \EndFor
   \end{algorithmic}
   \end{algorithm}

   \vspace{1em}
   \begin{hypothesis} \label{assumption:errors_and_step_sizes}
      The step sizes \((\gamma_k)_{k\in\N}\) are positive, the errors \((\varepsilon_k)_{k\in\N}\) are nonnegative, and the following assumptions hold:
   \begin{itemize}
      \item \(\sum_{k=1}^\infty \varepsilon_k < \infty\).
      \item \(\tfrac{\varepsilon_{k+1}}{\gamma_k} \sum_{i=0}^k \gamma_i  = O\left( \tfrac{1}{\sqrt{\sum_{i=0}^k \gamma_i }} \right) \).
      \item \(\sum_{k=1}^\infty \left( \left( \varepsilon_{k} / \gamma_k \right)^2 \sum_{i=0}^{k-1} \gamma_i \right) < \infty\).
      \item \(\sum_{k=1}^{\infty} \left( ( \varepsilon_{k} / \gamma_k ) \sum_{i=1}^{k}   \gamma_{i-1} \right)  < \infty\).
      \item \(\sum_{k=0}^\infty \gamma_k = \infty\).
   \end{itemize}
   \end{hypothesis}
   
   These assumptions ensure that the error decreases sufficiently fast to zero. For example, if the step sizes are constant \(\gamma_i = \gamma\),
   then the choice of error \( \varepsilon_k = \tfrac{1}{ ((k+1)\ln(k+1))^2} \) satisfies all the assumptions.
 The last hypothesis \(\sum_{k=0}^\infty \gamma_k = \infty\) is standard in the optimization literature.

 \vspace{1em}

\begin{theorem} \label{thm:IL}
   Assume that the sequences \((\gamma_{A,i} , \gamma_{B,i})_{i\in \N}\) and \( (\varepsilon_k)_{i\in \N}\) are such that Hypothesis~\ref{assumption:errors_and_step_sizes} is satisfied
   for both sequences of step sizes, and that \( \gamma_{A,i}/\gamma_{B,i} \to 0\). Assume moreover that \(C_B(\hat{u}) \in \KK_B \) is feasible and that the constraints have no recession direction as defined in \cite[Definition 3]{andrews2025augmented}. 
   Then the iterates of \Cref{algo:IALM_IL} converge to the solutions of the trilevel problem 
   \begin{align} \label{eq:closest_feasible_IL}
      \underset{\hat{u} \in \DD}{\mathrm{minimize}} &\quad  \mathcal{L}_{U^* , \Xi^*}( \hat u ) + r(\hat u) \\
      \text{s.t.} &\quad \hat{u} \in \underset{\tilde u \in \DD}{\mathrm{argmin}}\quad \left\Vert C_A(\tilde u) - \mathrm{Proj}_{\KK_A} (C_A(\tilde u)) \right\Vert \nonumber \\
         & \text{s.t.} \quad C_B(\hat{u}) \in \KK_B . \nonumber
   \end{align}
\end{theorem}

\begin{proof}
The proof is given in Appendix~\ref{app:proof_main_theorem}. 
\end{proof}

\vspace{1em}

\begin{remark}
   When both constraints \( C_A(\hat{u}) \in \KK_A \) and \( C_B(\hat{u}) \in \KK_B \) are feasible simultaneously, \eqref{eq:closest_feasible_IL2} is equal to \eqref{eq:constrained IL 2}.
    Indeed, in that case,
    \(\underset{\tilde u \in \DD}{\mathrm{min}}\quad \left\Vert C_A(\tilde u) - \mathrm{Proj}_{\KK_A} (C_A(\tilde u)) \right\Vert = 0\)
    simply means \(C_A(\hat{u}) \in \KK_A \).
    Thus, when problem \eqref{eq:constrained IL 2} is feasible, \Cref{algo:IALM_IL} simply solves it. 
\end{remark}

\section{Toy-car experiment} \label{sec:toy_car_experiment}

\begin{table*}[!t]
   \centering
   \caption{Comparison of pedestrian collision rates (\%) under different control/optimization schemes.}
   \label{table:accident_statistics}
   \resizebox{\textwidth}{!}{%
   \begin{tabular}{lcccc}
       \toprule
       Method & Expert Controller & Hierarchical IALM & Only Total Acceleration Constraint & Both Constraints at Same Level \\
       \midrule
       Collision Rate (\%) & 18\% & 0\% & 31\% & 14\% \\
       \bottomrule
   \end{tabular}
   }
\end{table*}

In autonomous driving, one may wish to impose both a maximum admissible acceleration and a constraint ensuring a safe distance from pedestrians.
We call constraint \(A\) the constraint related to maximum acceleration and constraint \(B\) the one related to pedestrian safety.
Such constraints allow one to learn from expert examples while favoring safe behavior.
However, a demonstration in which the driver brakes sharply to avoid a collision may be infeasible: respecting the maximum acceleration constraints forbids braking sharply, and 
respecting a safety constraint around pedestrians might require sharp braking, and therefore high acceleration, if the vehicle has significant speed and the pedestrian is detected late.
Discarding such data would not be a good solution, since they still contain expert behavior in critical scenarios; indeed, even an autonomous vehicle might encounter pedestrians who step onto the road unexpectedly.

We study a toy experiment in which a car, modeled as a kinematic bicycle, drives on a closed stadium-shaped track and a pedestrian may appear suddenly on a crosswalk. 
See \Cref{fig:toy_car_environment_and_rollouts} for a representation of the setting. 
The vehicle state is \(x=(p_x,p_y,\psi,v)\), where \((p_x,p_y)\) is the position of the vehicle, \(\psi\) its heading, and \(v\) its longitudinal speed. The control is \(u=(a,\tau)\),
 where \(a\) is the longitudinal acceleration and \(\tau=\tan(\delta)\) is the steering variable associated with the front-wheel angle \(\delta\). 
Expert demonstrations are generated by a pure-pursuit controller perturbed by Ornstein--Uhlenbeck noise in both control channels. 
When a pedestrian appears on the road at short distance, the expert switches to a braking mode with constant negative acceleration until the vehicle comes to a stop. 
The braking level is randomized, and not all expert demonstrations successfully maintain a safe distance from the pedestrian (see \Cref{table:accident_statistics}).
 
The learned policy is a neural network that receives the vehicle state together with a one-hot pedestrian-visibility input. 
In our experiments, the lower-priority constraint \(A\) is a total-acceleration bound, while the higher-priority constraint \(B\) is a pedestrian-braking constraint that is active only when the pedestrian is visible.
 Additional experimental details are given in Appendix~\ref{app:toy_car_details}.

More precisely, in the experiment we use a quadratic constraint on the controls \((\hat a_k,\hat \tau_k)\):
\[
C_A(\hat u) \in \KK_A
\quad\Longleftrightarrow\quad
\forall k,\;
\hat a_k^2 + \left(\frac{v_k^2}{L}\hat \tau_k\right)^2 \le a_{\mathrm{tot,max}}^2,
\]
where \(L\) denotes the wheelbase of the kinematic bicycle model.
Moreover,
\[
C_B(\hat u) \in \KK_B
\quad\Longleftrightarrow\quad
\forall k_p,\;
\hat a_{k_p} \le -\frac{v_{k_p}^2}{2d_{\mathrm{ped},{k_p}}},
\]
where \(k_p\) indexes the data points at which a pedestrian is visible.
\(d_{\mathrm{ped},k}\) denotes the free distance from the front of the vehicle to the pedestrian, including a safety margin.

We compare \Cref{algo:IALM_IL} with two natural alternatives. 
The three methods are:
\begin{itemize}
   \item Hierarchical IALM: \Cref{algo:IALM_IL} with \(\gamma_{A,i} = 5/(i+1) \) and  \(\gamma_{B,i} = 15 \)
   \item Only total-acceleration constraint: the pedestrian-braking constraint is ignored, and only the total-acceleration constraint is kept, with \(\gamma_{A,i} = 5/(i+1) \)
   \item Both constraints at the same level: the ALM is used with  \(\gamma_{A,i} = \gamma_{B,i} = 15 \). The two constraints are treated equally.
\end{itemize}

The results in \Cref{fig:toy_car_total_acceleration_histograms} indicate that Hierarchical IALM correctly prioritizes the pedestrian-safety constraint, yielding no collisions while minimizing the violation of the 
lower-priority constraint. This is consistent with the theory of \Cref{thm:IL}. By comparison, when only the acceleration constraint is kept and no pedestrian-safety constraint is imposed, the learned policy 
respects that constraint but produces more pedestrian collisions than the expert policy on which it was trained, since it cannot decelerate hard enough to avoid collision.
Finally, if one simply applies the usual IALM with identical penalty terms on all constraints, as in \cite{andrews2025augmented}, the result is a compromise in which neither constraint is respected perfectly, but 
the learned policy violates both constraints less than the expert. This is also consistent with the theory in \cite[Theorem 3.5]{andrews2025augmented}. 

\begin{figure}[t]
    \centering
    \begin{minipage}{0.48\linewidth}
        \centering
        \includegraphics[width=\linewidth]{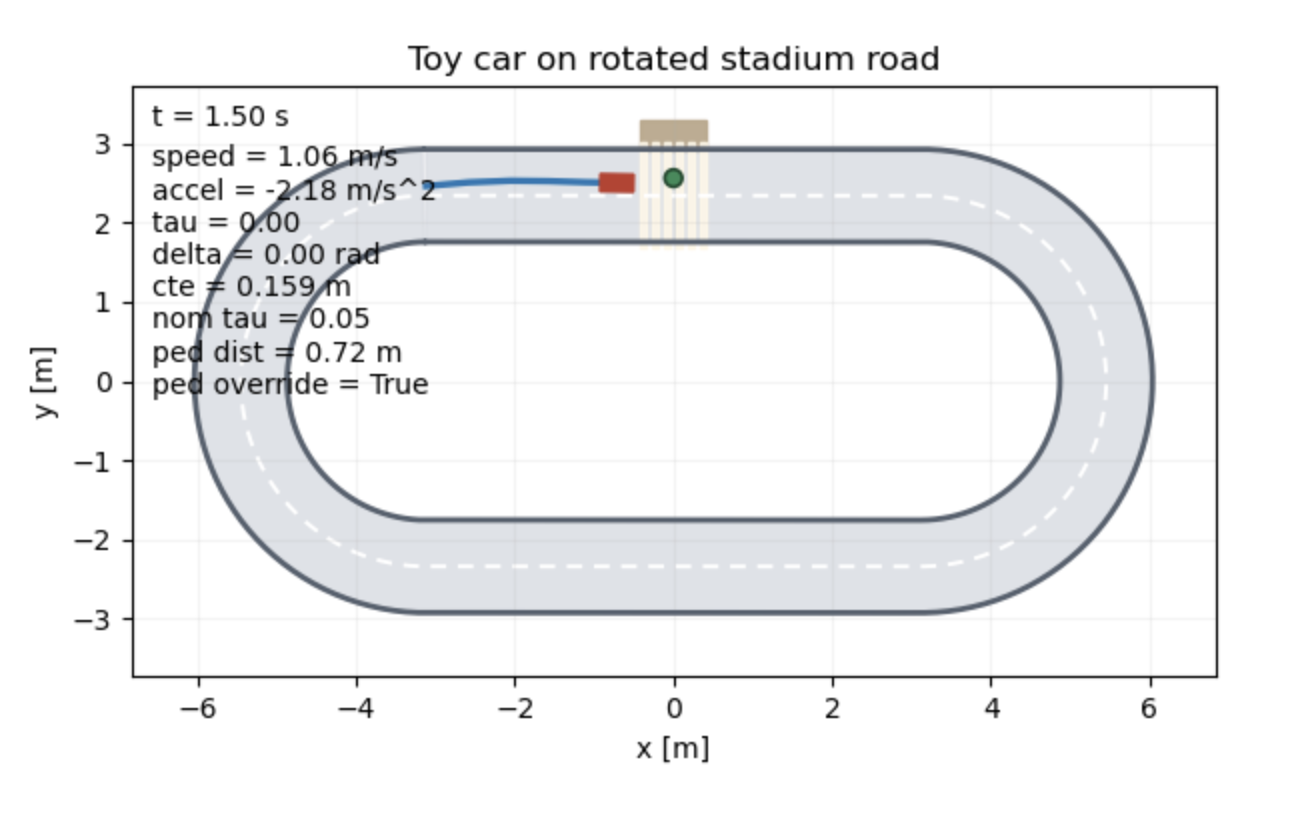}
    \end{minipage}\hfill
    \begin{minipage}{0.48\linewidth}
        \centering
        \includegraphics[width=\linewidth]{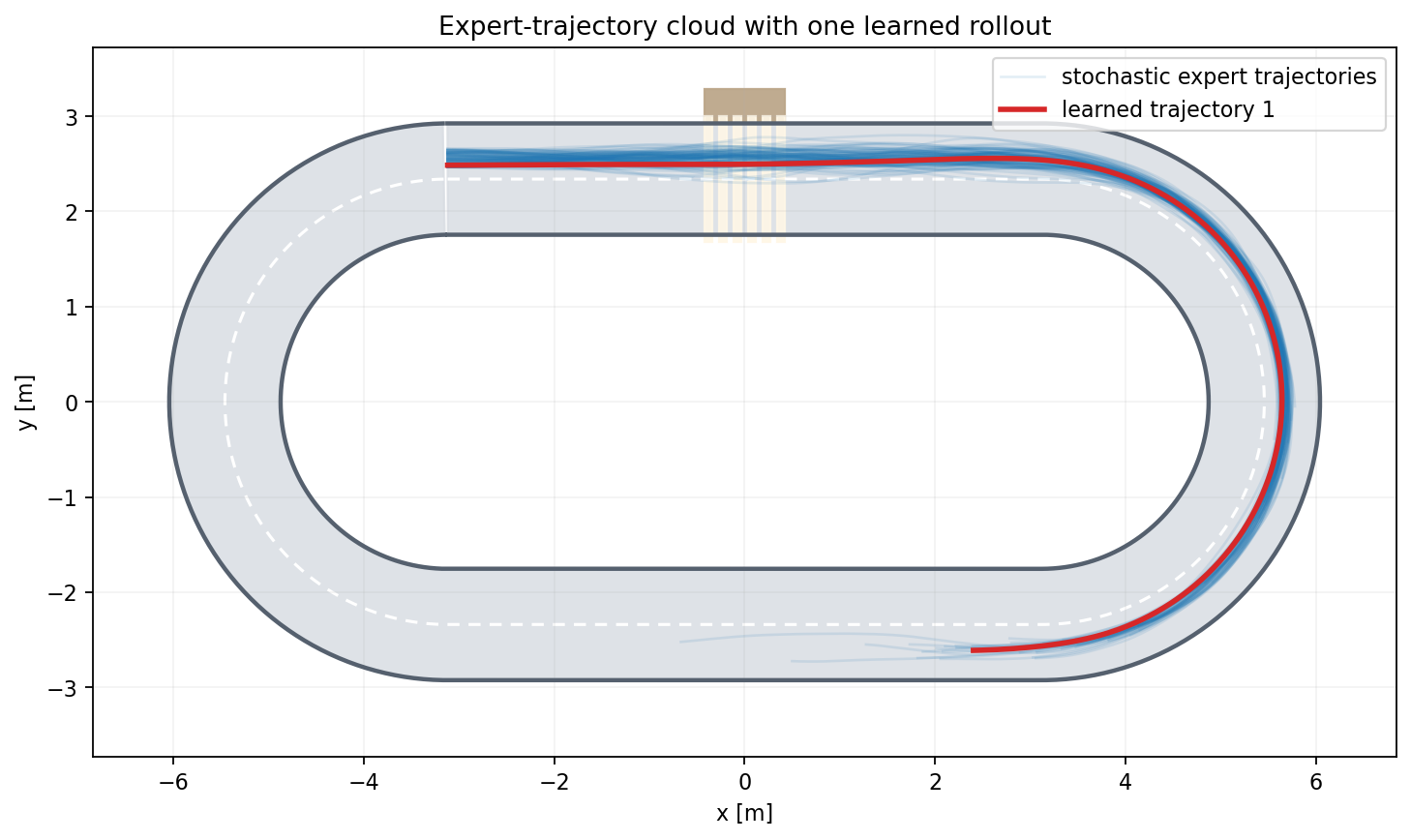}
    \end{minipage}
    \caption{Left: overview of the track together with the pedestrian (green disc) at the moment it appears on the road and an expert car (red rectangle) rollout. Right: cloud of expert trajectories (blue) together with one learned rollout (red), illustrating the path of the training data.}
    \label{fig:toy_car_environment_and_rollouts}
\end{figure}

\begin{figure}[t]
    \centering
    \begin{minipage}{0.32\linewidth}
        \centering
        \includegraphics[width=\linewidth]{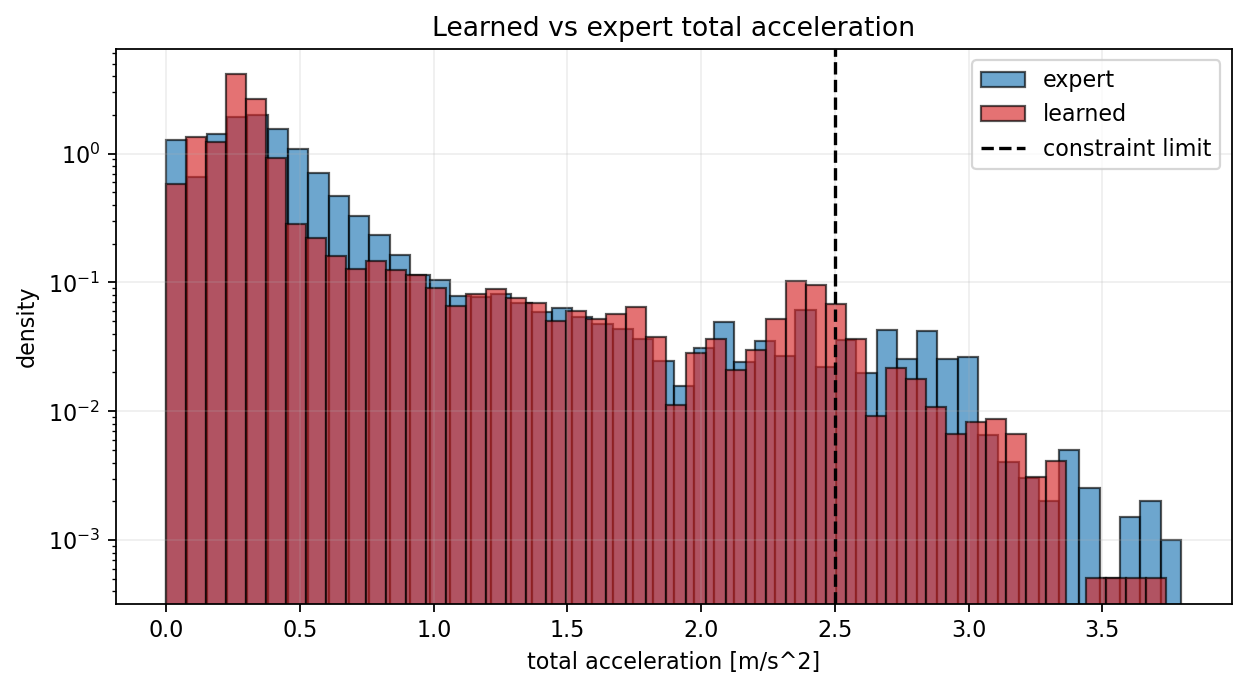}

        \small (a) Hierarchical IALM
    \end{minipage}\hfill
    \begin{minipage}{0.32\linewidth}
        \centering
        \includegraphics[width=\linewidth]{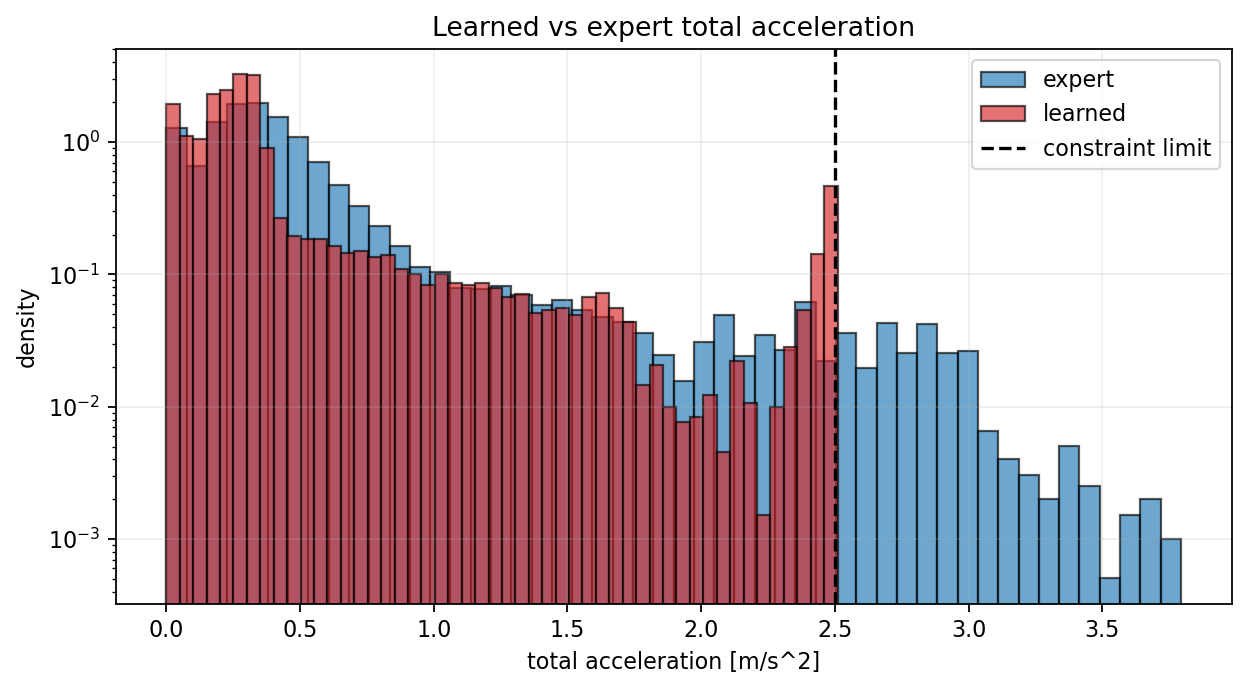}

        \small (b) Total-acceleration constraint only
    \end{minipage}\hfill
    \begin{minipage}{0.32\linewidth}
        \centering
        \includegraphics[width=\linewidth]{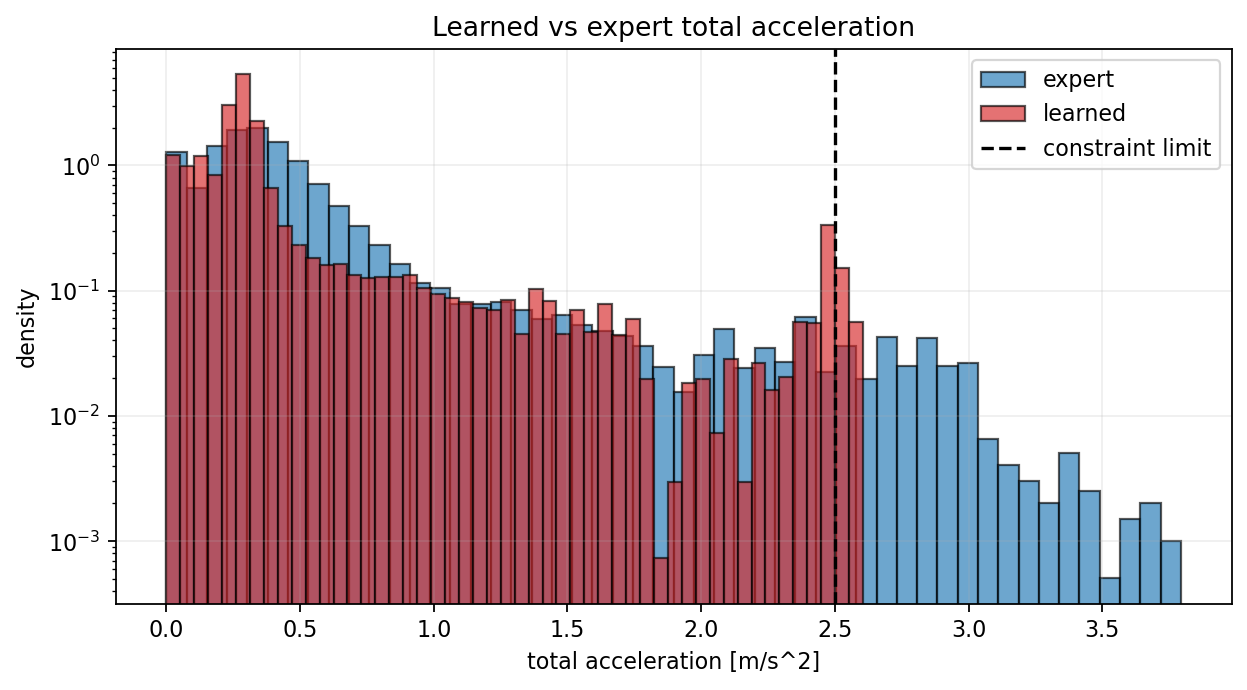}

        \small (c) Both constraints at the same level
    \end{minipage}
    \caption{Comparison of the learned and expert total-acceleration distributions under the three optimization schemes considered in the toy-car study. y-axis is log-scaled.}
    \label{fig:toy_car_total_acceleration_histograms}
\end{figure}
\section{CONCLUSION}
In this work, we study constrained imitation-learning problems in the infeasible case, when safety, robustness, and stability constraints may conflict with one another.
We argue that such infeasible settings arise in practice.
 We propose a two-penalty inexact augmented-Lagrangian method that provides guarantees on the behavior of the imitation-learning optimization problem and an analytically tractable limit for the learned policy. 
 We show that, in the convex setting, the method converges to a lexicographic closest-feasible problem: the high-priority constraints are enforced first when feasible, the lower-priority constraint violations are then minimized, and the imitation objective is optimized over that set.

The toy-car study illustrates this mechanism on a simple but representative example, where pedestrian safety and total-acceleration limits can conflict on emergency-braking demonstrations. 
In that setting, the trilevel hierarchy produces safer behavior than alternatives while remaining close to the expert distribution.
 These results suggest that infeasibility can be treated effectively in constrained imitation learning, and that hierarchical augmented-Lagrangian formulations are a promising way to do so.
  A natural next step is to extend the analysis to richer nonconvex policy classes and more realistic large-scale benchmarks.

\bibliography{references}

\appendices

\section{Proof of the main theorem}
\label{app:proof_main_theorem}

\begin{theorem}
   Assume that \((\gamma_{A,i} , \gamma_{B,i})_{i\in \N}\) and \( (\varepsilon_k)_{i\in \N}\) are such that Hypothesis~\ref{assumption:errors_and_step_sizes} is satisfied
   for both sequences of step sizes, and that \( \gamma_{A,i}/\gamma_{B,i} \to 0\). Assume moreover that \(C_B(\hat{u}) \in \KK_B \) is feasible and that the constraints have no recession direction as defined in \cite[Definition 3]{andrews2025augmented}. 
   Then the iterates of \Cref{algo:IALM_IL} converge to the solutions of the trilevel problem 
   \begin{align} \label{eq:closest_feasible_IL2}
      \underset{\hat{u} \in \DD}{\mathrm{minimize}} &\quad  \mathcal{L}_{U^* , \Xi^*}( \hat u ) + r(\hat u) \\
      \text{s.t.} &\quad \hat{u} \in \underset{\tilde u \in \DD}{\mathrm{argmin}}\quad \left\Vert C_A(\tilde u) - \mathrm{Proj}_{\KK_A} (C_A(\tilde u)) \right\Vert \nonumber \\
         & \text{s.t.} \quad C_B(\hat{u}) \in \KK_B . \nonumber
   \end{align}
\end{theorem}

\begin{proof}
The proof is a weighted version of the infeasible augmented-Lagrangian argument of \cite{andrews2025augmented}. The only additional point is that the two constraint blocks are penalized with different sequences, and the ratio
\[
\frac{\gamma_{A,i}}{\gamma_{B,i}} \to 0
\]
induces a lexicographic limit. In this appendix we also make explicit one assumption that is implicit in the statement of the theorem: to obtain exactly the condition \(C_B(\hat u)\in \KK_B\) in the limit problem, we assume that the set
\[
\{\hat u \in \DD \mid C_B(\hat u)\in \KK_B\}
\]
is nonempty. Otherwise, one would first obtain a closest-feasible problem for the \(B\)-block.

\paragraph{Step 1: residual formulation.}
Introduce residuals
\[
s_A = C_A(\hat u)-y_A,
\qquad
s_B = C_B(\hat u)-y_B,
\]
with \(y_A\in \KK_A\) and \(y_B\in \KK_B\), and define the attainable residual set
\[
\mathcal{S}
\triangleq
\left\{
(s_A,s_B)\;\middle|\;
\exists \hat u\in\DD,\;
s_A\in C_A(\hat u)-\KK_A,\;
s_B\in C_B(\hat u)-\KK_B
\right\}.
\]
Because the \(B\)-constraints are feasible, the slice
\[
\{s_A \mid (s_A,0)\in \overline{\mathcal S}\}
\]
is nonempty. We then define the lexicographically minimal residual set
\[
\mathcal{S}_{\mathrm{lex}}
\triangleq
\arg\min
\left\{
\|s_A\|
\;\middle|\;
(s_A,0)\in \overline{\mathcal S}
\right\}.
\]
Thus the trilevel problem of the theorem consists in enforcing first the exact feasibility of the \(B\)-block, then the smallest possible violation of the \(A\)-block, and finally the smallest value of the original objective \(f=\mathcal{L}_{U^*,\Xi^*}+r\).

\paragraph{Step 2: weighted closest-feasible residuals.}
For each iteration \(i\), let
\[
\bar s^i=(\bar s_A^i,\bar s_B^i)
\in
\arg\min_{(s_A,s_B)\in\overline{\mathcal S}}
\left(
\gamma_{A,i}\|s_A\|^2+\gamma_{B,i}\|s_B\|^2
\right).
\]
Pick any \(\tilde s_A\in \mathcal{S}_{\mathrm{lex}}\). Since \((\tilde s_A,0)\in\overline{\mathcal S}\), optimality of \(\bar s^i\) gives
\[
\gamma_{A,i}\|\bar s_A^i\|^2+\gamma_{B,i}\|\bar s_B^i\|^2
\leq
\gamma_{A,i}\|\tilde s_A\|^2.
\]
Dividing by \(\gamma_{B,i}\) yields
\[
\frac{\gamma_{A,i}}{\gamma_{B,i}}\|\bar s_A^i\|^2+\|\bar s_B^i\|^2
\leq
\frac{\gamma_{A,i}}{\gamma_{B,i}}\|\tilde s_A\|^2.
\]
Since \(\gamma_{A,i}/\gamma_{B,i}\to 0\), we immediately obtain
\[
\bar s_B^i \to 0.
\]
The same inequality shows that \((\bar s_A^i)\) is bounded. Any accumulation point \((s_A^\infty,s_B^\infty)\) of \((\bar s^i)\) therefore satisfies \(s_B^\infty=0\) and
\[
\|s_A^\infty\|\leq \|\tilde s_A\|.
\]
By minimality of \(\tilde s_A\), this implies \(s_A^\infty\in \mathcal{S}_{\mathrm{lex}}\). Therefore
\[
\bar s_B^i \to 0,
\qquad
\mathrm{Dist}(\bar s_A^i,\mathcal{S}_{\mathrm{lex}})\to 0.
\]
So the weighted closest-feasible residuals converge to the lexicographic residual set.

\paragraph{Step 3: the algorithm tracks the weighted residuals.}
The proof of \cite{andrews2025augmented} extends to the present setting by replacing the scalar penalty parameter with the block-diagonal matrix
\[
\Gamma_i
\triangleq
\begin{pmatrix}
\gamma_{A,i}I & 0\\
0 & \gamma_{B,i}I
\end{pmatrix}.
\]
Equivalently, the dual proximal-point interpretation is unchanged, except that the Euclidean quadratic is replaced by the weighted quadratic induced by \(\Gamma_i^{-1}\). Under the same assumptions on errors and stepsizes as in \cite[Hypothesis 1]{andrews2025augmented}, applied separately to \((\gamma_{A,i})\) and \((\gamma_{B,i})\), the residuals generated by the algorithm track the weighted closest-feasible vectors \(\bar s^i\). In particular, one gets the analog of the one-penalty residual-tracking estimate
\[
\|s_A^i-\bar s_A^i\| \to 0,
\qquad
\|s_B^i-\bar s_B^i\| \to 0,
\]
and therefore
\[
s_B^i \to 0,
\qquad
\mathrm{Dist}(s_A^i,\mathcal{S}_{\mathrm{lex}})\to 0.
\]
This is the point at which the condition \(\gamma_{A,i}/\gamma_{B,i}\to 0\) becomes decisive: the algorithm follows the weighted closest-feasible residuals, and those residuals themselves converge to a lexicographic limit.

\paragraph{Step 4: boundedness of the primal sequence.}
The no-recession-direction assumption is exactly the assumption used in \cite{andrews2025augmented} to obtain level boundedness with respect to perturbations of the constraints. Since the residual sequence \((s_A^i,s_B^i)\) remains eventually in a bounded neighborhood of \(\mathcal{S}_{\mathrm{lex}}\times\{0\}\), and since the same dual estimates as in \cite{andrews2025augmented} keep the objective values \(f(\hat u^i)\) eventually bounded from above, the sequence \((\hat u^i)\) is bounded.

\paragraph{Step 5: identification of accumulation points.}
Take a convergent subsequence \(\hat u^{i_k}\to \hat u^\infty\). Because
\[
s_B^{i_k}\to 0,
\qquad
\mathrm{Dist}(s_A^{i_k},\mathcal{S}_{\mathrm{lex}})\to 0,
\]
and because the residual graph is closed under our standing convexity and closedness assumptions, we obtain
\[
C_B(\hat u^\infty)\in \KK_B
\]
and
\[
C_A(\hat u^\infty)-\mathrm{Proj}_{\KK_A}(C_A(\hat u^\infty))
\in
\mathcal{S}_{\mathrm{lex}}.
\]
Hence \(\hat u^\infty\) is feasible for the first two levels of the trilevel problem. It remains to identify the value of the objective on this lexicographically feasible set. This is done exactly as in the one-penalty argument of \cite{andrews2025augmented}: lower semicontinuity of \(f\), together with residual convergence and boundedness, implies that every accumulation point attains the minimum value of \(f\) over the lexicographically feasible set.

\paragraph{Conclusion.}
Every accumulation point of \((\hat u^i)\) is therefore a solution of \eqref{eq:closest_feasible_IL2}. Since the sequence is bounded, if
\[
\mathrm{Dist}\!\left(\hat u^i,\mathrm{Sol}\!\left(\eqref{eq:closest_feasible_IL2}\right)\right)
\]
did not converge to zero, one could extract a subsequence staying at a positive distance from the solution set, and this subsequence would still admit an accumulation point. But every accumulation point belongs to the solution set, which is a contradiction. Therefore the whole sequence converges to the solution set of the trilevel problem.
\end{proof}

\section{Additional details on the toy-car benchmark}
\label{app:toy_car_details}
The toy-car environment is based on a kinematic bicycle system with state \((x,y,\psi,v)\) and controls \((a,\tau)\), where \(\tau=\tan(\delta)\). The road is a closed stadium-shaped track composed of two straights and two semicircular turns.

\paragraph{Expert demonstrations.}
The nominal expert is a pure-pursuit controller with curvature-dependent speed tracking. Its controls are perturbed by Ornstein--Uhlenbeck noise in both acceleration and steering channels. A pedestrian event is sampled independently for each rollout. When the event is active and the car reaches the trigger region near the crosswalk, the pedestrian becomes visible for \(3\) seconds and the expert switches to an emergency override: the steering command is set to zero and the longitudinal acceleration is replaced by a constant sampled braking value in the interval \([-3.0,-2.0]\), until the speed drops below a small stop threshold.

\paragraph{Dataset and policy class.}
The experiments use \(1000\) expert trajectories with time step \(0.05\) s and dataset horizon \(16\) s. Initial conditions are sampled near the outer lane, with fixed initial progress, lateral jitter \(0.2\), heading jitter \(0.1\) rad, and initial speed uniformly sampled in \([1.2,1.9]\). The learned policy is a multilayer perceptron with hidden widths \([512,512,512,512]\) and \(\tanh\) activations. Its input is the normalized vehicle state together with a one-hot encoding of pedestrian visibility, and its output is the control pair \((a,\tau)\).

\paragraph{Constraints and training.}
In the present experiments, the lower-priority constraint is the total-acceleration residual
\[
\hat a^2 + \left(\frac{v^2}{L}\hat \tau\right)^2 - a_{\mathrm{tot,max}}^2 \le 0,
\]
with \(a_{\mathrm{tot,max}}=2.5\). The higher-priority constraint is a pedestrian-braking residual based on the approximate free distance between the midpoint of the front edge of the car and the pedestrian disk, with a \(10\%\) radius margin and a clipped braking bound. This second constraint is enforced only on visible-pedestrian samples. The experiments do not include any lane-boundary CBF term or steering-centering loss. Training uses batch size \(2000\), \(15\) outer augmented-Lagrangian iterations, \(5000\) inner optimization steps per outer iteration, \(\gamma_{A,i}=5/(i+1)\), \(\gamma_{B,i}=15\), and weight decay \(10^{-7}\).

\paragraph{Evaluation protocol.}
We compare the trilevel hierarchy, the total-acceleration-only baseline, and the flat two-constraint baseline. Evaluation is performed by closed-loop rollouts of the learned policies on randomized initial conditions close to the nominal lane. The reported results use \(2000\) rollouts of horizon \(13\) s. We report pedestrian collision rates and complement them with qualitative trajectory plots and control or acceleration histograms, shown in the main text and supplementary figures.

\section{Supplementary figures}
\label{app:planned_figures}

Supplementary figures can be found \Cref{fig:toy_car_control_histograms}.
\begin{figure}[t]
    \centering
    \includegraphics[width=0.7\linewidth]{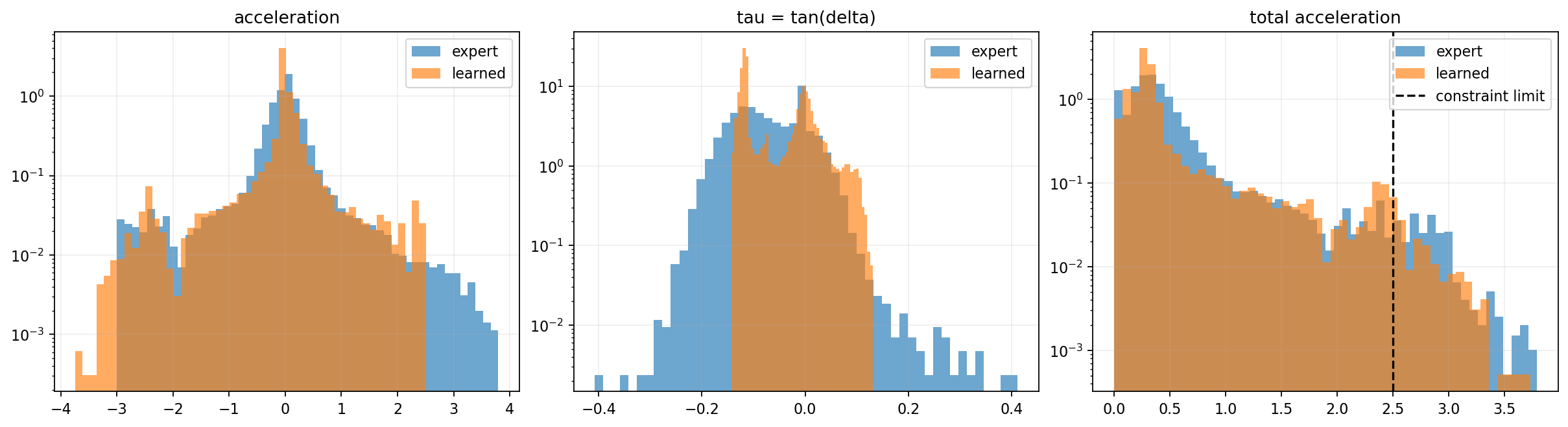}
    
    \small (a) Hierachical IALM
    \vspace{1em}

    \includegraphics[width=0.7\linewidth]{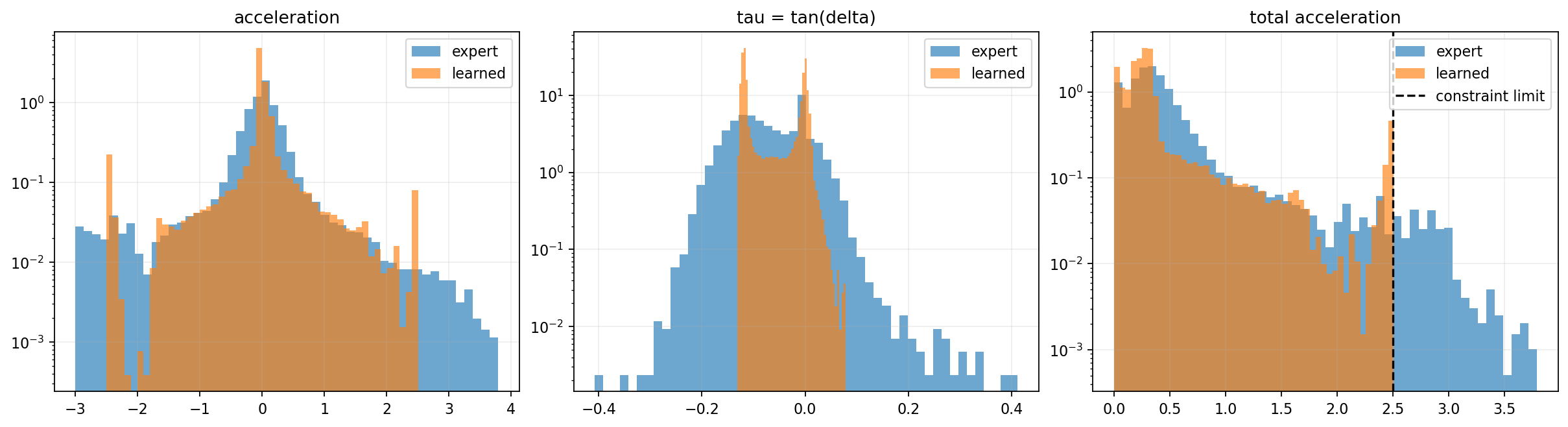}
    
    \small (b) Only total-acceleration constraint
    \vspace{1em}

    \includegraphics[width=0.7\linewidth]{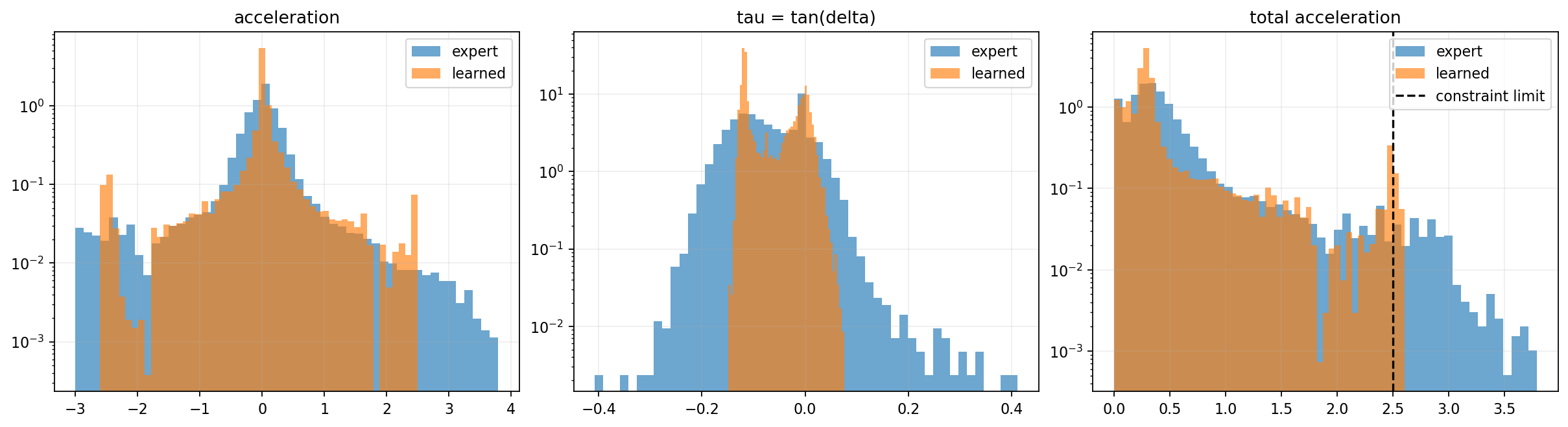}

        \small (c) Both constraints at the same level

    \caption{Control histograms for the expert and learned policies under the three optimization schemes. These supplementary plots help interpret how the hierarchy between pedestrian-avoidance and acceleration constraints changes the control distribution produced by the learned policy.}
    \label{fig:toy_car_control_histograms}
\end{figure}

\end{document}